\documentclass[11pt]{article}
\usepackage[colorlinks]{hyperref}            
\usepackage{times,comment}
\usepackage{color,graphicx,amsthm,amssymb,amsmath,cite,xspace,setspace}
\usepackage[small,bf]{caption} \usepackage{multicol, multirow}
\definecolor{orange}{rgb}{.6,0.1,.6}

\usepackage{graphicx}
\usepackage{soul}
\usepackage{algorithm,algorithmic}
\usepackage{colortbl}
\usepackage{rotating}
\usepackage{times}
\usepackage{textcomp}
\usepackage{url}

\usepackage{graphicx}
\usepackage{epsf}
\usepackage{wasysym}
\usepackage{mdwtab}
\usepackage{cite}
\usepackage{algorithmic}
\usepackage{algorithm}

\usepackage{algorithmic}
\usepackage{algorithm}
\usepackage{bm}

\advance\headheight -\headheight
\advance\headsep -\headsep
\def \by{\mbox{\boldmath$y$}}
\def \bx{\mbox{\boldmath$x$}}

\def\R{{\mathbb R}}
\def\E{{\mathbb E}}

\newtheorem{theorem}{Theorem}[section]

\newtheorem{lemma}{Lemma}[section]

\newtheorem{corollary}{Corollary}[section]
\newtheorem{example}{Example}

\setstcolor{green}
\setul{}{0.3ex}

\def\sign{\mathop{\mbox{sign}}}
\def \bx{\boldsymbol{x}}
\def \bnu{\boldsymbol{\nu}}
\def \by{\boldsymbol{y}}
\def \bz{\boldsymbol{z}}
\def \bu{\boldsymbol{u}}
\def \bv{\boldsymbol{v}}
\def \bb{\boldsymbol{b}}
\def \bg{\boldsymbol{g}}
\def \ba{\boldsymbol{a}}
\def \bA{\boldsymbol{A}}
\def \bB{\boldsymbol{B}}
\def \bC{\boldsymbol{C}}

\def\R{{\mathbb R}}
\def\E{{\mathbb E}}

\def \K{{\cal K}}

\usepackage{soul}

\begin{document}

\title{\bf {\Large Sparse Estimation with Strongly Correlated Variables \\ using  Ordered Weighted $\ell_1$ Regularization }}
\author{  {\sl M\'{a}rio A. T. Figueiredo}\\
  Instituto de Telecomunica\c{c}\~{o}e and
  Instituto Superior T\'{e}cnico, Universidade de Lisboa, Portugal\\ \ \\
 \and
{\sl Robert D. Nowak}\\
Department of Electrical and Computer Engineering, University of Wisconsin-Madison, USA\\ \ \\
}

\maketitle
\begin{abstract}
This paper studies {\it ordered weighted} $\ell_1$ (OWL) norm regularization for sparse estimation problems with strongly correlated variables.  We prove sufficient conditions for clustering based on the correlation/colinearity of variables using the OWL norm, of which the so-called OSCAR \cite{bondell2007} is a particular case. Our results extend previous ones for OSCAR in several ways: for the squared error loss, our conditions hold for the more general OWL norm and under weaker assumptions; we also establish clustering conditions for the absolute error loss, which is, as far as we know, a novel result. Furthermore, we characterize the statistical performance of OWL norm regularization for generative models in which certain clusters of regression variables are strongly (even perfectly) correlated, but variables in different clusters are uncorrelated.  We show that if the true $p$-dimensional signal generating the data involves only $s$ of the clusters, then $O(s\log p)$ samples suffice to accurately estimate the signal, regardless of the number of coefficients within the clusters.  The estimation of $s$-sparse signals with completely independent variables requires just as many measurements.  In other words,
using the OWL we pay no price (in terms of the number of measurements) for the presence of strongly correlated variables.
\end{abstract}

\pagenumbering{arabic}
\def\thepage{\arabic{page}}

\section{Introduction}
\subsection{Definitions, Problem Formulation, and Preview of Main Results}
The OWL ({\it ordered weighted} $\ell_1$) regularizer is defined as
\begin{equation}\label{OSCAR2}
\Omega_{\bm w} ({\bm x}) =  \sum_{i=1}^p w_i \, |x|_{[i]},
\end{equation}
where  $|x|_{[i]}$ is the $i$-th largest component in magnitude of $\bm x \in \mathbb{R}^p$,  and ${\bm w} \in \mathbb{R}_+^p$ is a vector of non-negative weights. If $w_1 \geq w_2 \geq \cdots \geq w_p$ and $w_1 >0$ (which we will assume to be always true), then $\Omega_{\bm w}$ is a norm (as shown in \cite{bogdan2014}, \cite{ZengFigueiredo2013}), which satisfies  $w_1 \|\bm x\|_{\infty}\leq \Omega_{\bm w}(\bm x) \leq w_1 \|\bm x\|_1.$ The OWL regularizer generalizes the OSCAR ({\it octagonal shrinkage and clustering algorithm for regression}) \cite{bondell2007}, which is obtained by setting $w_i = \lambda_1 + \lambda_2 \, (p-i)$, where $\lambda_1, \, \lambda_2 \geq 0$. Notice also that if $w_1 > 0$, and $w_2 = \cdots = w_p = 0$, the OWL is simply ($w_1$ times) the $\ell_{\infty}$ norm, whereas for $w_1 = w_2 = \cdots = w_p$, the OWL becomes ($w_1$ times) the $\ell_1$ norm.

In this paper, we will study the use of the OWL norm as a regularizer in linear regression with strongly correlated variables, both under the squared error loss and the absolute error loss, {\it i.e.}, the two  following optimization problems:
\begin{equation}\label{OWL_L2}
\min_{{\bm x}\in \mathbb{R}^p} \frac{1}{2}\| {\bm A}\, {\bm x} - {\bm y} \|_2^2 + \Omega_{\bm w}({\bm x}),
\end{equation}
where $\bm A \in \mathbb{R}^{n\times p}$ is the design matrix, and
\begin{equation}\label{OWL_L1}
\min_{{\bm x}\in \mathbb{R}^p} \| {\bm A}\, {\bm x} - {\bm y} \|_1 + \Omega_{\bm w}({\bm x}).
\end{equation}
We also consider constrained versions of these problems; see (\ref{c2}) and (\ref{c1}) below.

The first of our two main results gives sufficient conditions for OWL norm regularization (with either the squared or the absolute error loss) to automatically cluster strongly correlated variables, in the sense that the coefficients associated with such variables have equal estimated values (in magnitude). The result for the squared error loss extends the main theorem about OSCAR in \cite{bondell2007}, since not only it applies to the more general case of OWL (of which OSCAR is a particular case), but it also holds under weaker conditions. Furthermore, the result for the absolute error loss is, as far as we know, novel.

Our second main result is a finite sample bound for the OWL regularization procedure, which includes the standard LASSO and OSCAR as special cases.  To the best of our knowledge, these are the first finite sample error bounds for sparse regression with strongly correlated columns.
To preview this result, consider the following special case (which we generalize further in the paper): assume that we observe
\begin{eqnarray}
\by & = & \bA \bx^\star \ + \ \bnu \ .
\label{obs}
\end{eqnarray}
where $\bnu \in \R^n$ is the measurement error satisfying
\begin{eqnarray}
\frac{1}{n} \|\bnu\|_1 \ \leq \ \varepsilon\, ,
\end{eqnarray}
and about which we make no other assumptions.  The measurement/design matrix $\bA$ is Gaussian distributed.  For the purposes of this introduction, assume that each column of $\bA$ has i.i.d.\ ${\cal N}(0,1)$ entries, but that the columns may be correlated.  Specifically, assume that the columns can be grouped so that columns within each group are identical (apart from a possible sign flip) and columns in different groups are uncorrelated.  This models cases in which certain variables are perfectly correlated with each other, but uncorrelated with all others. The vector $\bx^\star\in \R^p$ is assumed to satisfy $\|\bx^\star\|_1\leq \sqrt{s}$.  Note, for example, that this condition is met if $\|\bx^\star\|_2 \leq 1$ and $\bx^\star$  has  at most $s$ non-zero components (is $s$-sparse), which we assume to be true. Note that since certain columns of $\bA$ may be identical, in general there may be many sparse vectors  $\bx$ such that $\bA\bx = \bA\bx^\star$.
Thus, for now, assume that if two columns of $\bA$ are identical (up to a sign flip), then so are (in magnitude) the corresponding coefficients in $\bx^\star$. The following theorem essentially shows that the number of measurements sufficient to estimate an $s$-sparse signal (i.e., a signal with $s$ nonzero groups of identical coefficients corresponding to identical columns in $\bA$), with a given precision, grows like
\begin{equation}
n \ \sim \ s\log p\, .
\end{equation}
This agrees with well-known sample complexity bounds for sparse recovery under stronger assumptions such as the restricted isometry property or i.i.d.\ measurements \cite{candes06,donoho06,haupt06,candes07,vershynin14}. Moreover, this shows that {\em by using OWL we pay no price (in terms of the sufficient number of measurements) for colinearity of some columns of} $\bA$.

\begin{theorem}
Let $\by$, $\bA$, $\bx^\star$, and $\varepsilon$ be as defined above. Let $\Delta := \min \{ w_l - w_{l+1},\; l=1,...,p-1 \}$ be the minimum gap between two consecutive components of vector $\bm w$, and assume $\Delta > 0$. Let $\widehat{\bx}$ be a solution to either of the two following optimization problems:
\begin{eqnarray}
\min_{\bx \in \R^p} \Omega_{\bm w}(\bx) \mbox{ \ subject to \ } \frac{1}{n}\|\bA\bx-\by\|_2^2 \ \leq \ \varepsilon^2 ,
\label{c2}
\end{eqnarray}
or
\begin{eqnarray}
\min_{\bx \in \R^p} \Omega_{\bm w}(\bx)  \mbox{ \ subject to \ } \frac{1}{n}\|\bA\bx-\by\|_1  \ \leq \  \varepsilon .
\label{c1}
\end{eqnarray}
Then,
\begin{description}
\item[(i)] for every pair of columns $(i,j)$ for which\footnote{It is trivial to extend the proof of this result to show that if ${\bm a}_i =  -{\bm a}_j$, then  $\widehat{x}_i = -\widehat{x}_j$} ${\bm a}_i =  {\bm a}_j$, we have $\widehat{x}_i = \widehat{x}_j$;
\item[(ii)] the solution $\widehat \bx$ satisfies
\begin{eqnarray}
\E \, \|\widehat{\bx} -\bx^\star\|_2 & \leq & \sqrt{2\pi} \left(  4\sqrt{2} \, \frac{w_1}{\bar w} \, \sqrt{\frac{s \log p}{n}}  +\varepsilon\right)  ,
\label{e1}
\end{eqnarray}
where $\bar w = p^{-1} \sum_{i=1}^p w_i$.
\end{description}
\label{thm_preview}
\end{theorem}

{The expectation above (and elsewhere in the paper) is with respect to the Gaussian distribution of $\bA$.}
Part (i) of this theorem is proved in Subsection~\ref{pi}, and part (ii) in Subsection \ref{sec:main_and_corollaries}.
As mentioned above, in general there may be many sparse $\bx$ that yield the same value of $\bA\bx$.  This is where the OWL norm becomes especially important. If the columns are colinear, then the OWL solution will select a representation including all the columns associated with the true model, rather than an arbitrary subset of them.  {We generalize Theorem~\ref{thm_preview} in the paper, and show the OWL norm yields similar clustering and recovery conditions for problems with strongly correlated, but not necessarily colinear, variables.} Notice that the constant factor in the bound is typically a small constant;  for example, in the OSCAR case we have $w_i = \lambda_1+\lambda_2\, (p-i)$, thus $\bar w = \lambda_1+\lambda_2 \, (p-1)/2$ and therefore $w_1/ \bar w \leq 2$.

\subsection{Related Work}
Estimates obtained with the LASSO ({\it i.e.}, $\ell_1$) regularizer can be difficult to interpret when columns of the measurement matrix $\bA$ are strongly correlated, because it may select only one of a group of highly correlated columns. For scientific and engineering purposes, one is often interested in identifying all of the columns that are important for modeling the data, rather than just a subset of them.  Many researchers have proposed alternatives to the LASSO that aim at dealing with this problem.  For example, Jia and Yu \cite{yu10} study the {\em elastic net} regularizer (a combination of the $\ell_1$ and the squared $\ell_2$ norms), showing that it can consistently select the true model for certain correlated design matrices $\bA$, when LASSO cannot.  Marginal regression methods have also been shown by  Genovese {\it et al} to perform better than the LASSO, in the presence of strongly correlated columns \cite{jin12}.  Stability selection procedures, can also aid in the selection of correlated columns, as shown by Meinshausen and B\"{u}lhmann \cite{mein10}, and Shah and Samworth \cite{sam13}. Recently, B\"{u}hlmann {\it et al} \cite{buhlmann13} proposed and analyzed a two-stage approach called {\em cluster-LASSO}, which first identifies clusters of correlated columns, then groups them, and finally applies LASSO or group-LASSO to the groups;  the cluster-LASSO is shown to be statistically consistent in certain cases. Adaptive grouping methods based on nonconvex optimizations have also been proposed by Shen and Huang \cite{huang10}, and shown to be asymptotically consistent under certain conditions.

Most closely related to this paper, is the so-called OSCAR ({\it octagonal shrinkage and clustering algorithm for regression}), proposed and analyzed by Bondell and Reich \cite{bondell2007}.  As mentioned above, OSCAR is a special case of the OWL regularizer, obtained with $w_i = \lambda_1 + \lambda_2 (p-i)$. The OSCAR method has been shown to perform well in practice, but prior work has not addressed its statistical consistency or convergence properties. Motivated by this formulation of OSCAR, the OWL regularizer was recently proposed by Zeng {\it et al} \cite{ZengFigueiredo2013}, as a generalization thereof. The OWL norm was also independently proposed by Bodgan {\it et al} \cite{candes13,bogdan2014}, who set the weights to $w_i = F^{-1}(1-iq/(2p))$, where $F$ is the cumulative distribution function of the error variables, and  $0<q<1$ is a parameter. Those authors showed that if ${\bm A}$ is orthogonal,  the solution to (\ref{OWL_L2}) with these weights has a  false discovery rate for variable selection bounded by $q(p-k)/p$, where $k$ is the number of non-zero coefficients in the true ${\bm x}$ that generates ${\bm y}$.

On the computational side, a key tool for solving problems of the form
\eqref{OWL_L2}, \eqref{OWL_L1}, \eqref{c2}, or \eqref{c1}, is the Moreau proximity operator of $\Omega_{\bm w}$ \cite{CombettesBauschke}, defined as
\[
\mbox{prox}_{\Omega_{\bm w}} (\bm u) = \arg\min_{\bm x} \frac{1}{2} \|\bm x - \bm u\|_2^2 + \Omega_{\bm w}(\bm x).
\]
Efficient $O(p \log p)$ algorithms to compute $\mbox{prox}_{\Omega_{\bm w}}$ have been recently  proposed by Bodgan {\it et al} \cite{candes13,bogdan2014}, and by Zeng {\it et al} \cite{ZengFigueiredo2013}, who generalize to the OWL case the algorithm proposed by Zhong and Kwok \cite{zhong2012efficient}. Even more recently, Zeng {\it et al} \cite{ZengFigueiredo2014} have show how $\Omega_{\bm w}$ can be written explicitly as an {\it atomic norm} (see \cite{chandrasekaran2012convex}, for definitions), opening the door to the efficient use of the conditional gradient (also known as Frank-Wolfe) algorithm \cite{jaggi2013revisiting}.

\subsection*{Notation}
We denote (column) vectors by lower-case bold letters, {\it e.g.}, $\bm x$, $\bm y$,  their transposes by $\bm x^T$, $\bm y^T$, the corresponding $i$-th and $j$-th components as $x_i$ and $y_j$, and
matrices by upper case bold letters, {\it e.g.}, $\bm A$, $\bm B$. A vector with all elements equal to 1 is denoted as $\bm 1$ and $|{\bm x}|$ denotes the vector with the absolute values of the components of $\bm x$. Given some vector $\bm x$, $x_{[i]}$ is its $i$-th largest component ({\it i.e.}, for ${\bm x}\in \mathbb{R}^p$,  $x_{[1]} \geq x_{[2]} \geq \cdots \geq x_{[p]}$, with ties broken by some arbitrary rule); consequently, $|x|_{[i]}$ is the $i$-th largest component of $x$ in magnitude. The vector obtained by sorting (in non-increasing order) the components  of ${\bm x}$ is denoted as ${\bm x}_{\downarrow}$, thus $|{\bm x}|_{\downarrow}$ denotes the vector obtained by sorting the components of ${\bm x}$ in non-increasing order of magnitude (allowing to write $\Omega_{\bm w}(\bm x) = \bm w^T |{\bm x}|_{\downarrow}$).

\section{Sufficient Conditions for OWL Variable Clustering}\label{sec:sufficient_clustering}
\subsection{Introduction}
In this section, we study the solutions of \eqref{OWL_L2} and \eqref{OWL_L1} in the case where the design matrix $\bm A$ has strongly correlated columns, and give corollaries for the particular case of OSCAR. The results presented below extend the main theorem of \cite{bondell2007} in several ways: for the squared error loss, our result applies to the more general case of the OWL (of which OSCAR is a particular case) and it holds under weaker conditions; the result for the absolute error regression case is, as far as we know, novel.

\subsection{Squared Error Loss with Correlated Columns}
Consider the regression problem \eqref{OWL_L2}, and let ${\bm a}_i\in \mathbb{R}^n$ denote the $i$-th column (for $i=1,...,p$) of matrix $\bm A$.  The following theorem shows that \eqref{OWL_L2} clusters (in the sense that the corresponding components of the solution are exactly equal in magnitude) the columns that are correlated enough.

\begin{theorem}\label{theo:correlated}
Consider the objective function in \eqref{OWL_L2} and assume, as is common practice in linear regression, that the columns of the matrix are normalized to a common norm, that is, $\|{\bm a}_k\|_2 = c$, for $k=1,...,p$. Let $\widehat{\bm x}$ be any minimizer of the objective function in \eqref{OWL_L2}. Then, for every pair of columns $(i,j)$ for which $\|{\bm y}\|\; \|\sign(\widehat{x}_i)\,{\bm a}_i - \sign(\widehat{x}_j)\,{\bm a}_j \|_2 < \Delta$ (where $\Delta := \min \{ w_l - w_{l+1},\; l=1,...,p-1 \}$ is the minimum gap between two consecutive components of vector $\bm w$), we have $|\widehat{x}_i| = |\widehat{x}_j|$.
\end{theorem}

Notice that if two columns (affected by the signs of the corresponding regression coefficients) are identical, {\it i.e.}, if $\|\sign(\widehat{x}_i)\,{\bm a}_i - \sign(\widehat{x}_j)\,{\bm a}_j \|_2 = 0$, any strictly positive value of $\Delta$ is sufficient to guarantee that these two columns will be clustered, that is, that the corresponding coefficients will be equal in magnitude.

The following corollary addresses the case where the columns of $\bm A$ have zero mean and unit norm.

\begin{corollary}\label{cor:normalized}
Let the columns of $\bm A$ be normalized to zero sample mean and unit norm: ${\bm 1}^T {\bm a}_k = 0$ and $\| {\bm a}_k\|_2 = 1$, for $k=1,...,p$. Denote their inner products ({\it i.e.}, the sample correlation of the corresponding explanatory variables) as $\rho_{ij} = {\bm a}_i^T {\bm a}_j / (\|{\bm a}_i\|_2 \, \|{\bm a}_j\|_2)= {\bm a}_i^T {\bm a}_j$. Then, the condition in Theorem \ref{theo:correlated} becomes $\|{\bm y}\|_2 \, \sqrt{2-2\, \rho_{ij}\, \sign(\widehat{x}_i\,\widehat{x}_j)} < \Delta$.
\end{corollary}

\begin{proof}
The corollary results trivially from inserting the normalization assumption and the definition of $\rho_{ij}$ into the equality $\|\sign(\widehat{x}_i)\, \bm a_i - \sign(\widehat{x}_j)\, \bm a_j\|_2 = \sqrt{(\sign(\widehat{x}_i)\, \bm a_i - \sign(\widehat{x}_j)\, \bm a_j)^T (\sign(\widehat{x}_i)\, \bm a_i - \sign(\widehat{x}_j)\, \bm a_j)}$.
\end{proof}

The following corollary results from observing that, in the OSCAR case, $w_i = \lambda_1 + \lambda_2 (p-i)$, thus $\Delta = \lambda_2$.

\begin{corollary}\label{cor:OSCAR}
In the particular case of OSCAR, and for the case of normalized columns (as in Corollary \ref{cor:normalized}), the condition  is
$\|{\bm y}\|_2\; \sqrt{2 - 2 \, \rho_{ij}\, \sign(\widehat{x}_i\,\widehat{x}_j)} < \lambda_2$.
\end{corollary}

Corollary \ref{cor:OSCAR} is closely related to Theorem~\ref{theo:correlated} of \cite{bondell2007}, but has weaker conditions: unlike in \cite{bondell2007}, our result does not require that both $x_i$ and $x_j$ are different from zero and from all other $x_k$, for $k\neq i,j$. Furthermore, Theorem~\ref{theo:correlated} applies to the more general class of OWL norms, not just to OSCAR.
Note also that the results in \cite{bondell2007} assume that columns are signed so that $\widehat{x}_i \geq 0$ for all $i$.  Our result could also be stated with this assumption, in which case $\|\sign(\widehat{x}_i)\,{\bm a}_i - \sign(\widehat{x}_j)\,{\bm a}_j \|_2$ simplifies to $\|{\bm a}_i - {\bm a}_j \|_2$. Finally, observe that, in the extreme case of perfectly correlated columns ($\rho_{ij}\, \sign(\widehat{x}_i\,\widehat{x}_j) = 1$), the condition for OSCAR simplifies to $\lambda_2 > 0$.

\subsection{Absolute Error Loss with Similar Columns}
Consider the regression problem under absolute error loss in \eqref{OWL_L1}. The following theorem shows that, also in this case, the OWL regularizer clusters (in the sense that the corresponding components of the solution are exactly equal in magnitude) the columns that are similar enough.

\begin{theorem}\label{theo:correlatedL1}
Let $\widehat{\bm x}$ be any minimizer of the objective function in \eqref{OWL_L1}.  Then, for every pair of columns $(i,j)$ for which $ \|\sign(\widehat{x}_i)\,{\bm a}_i - \sign(\widehat{x}_j)\,{\bm a}_j \|_1 < \Delta $ (where $\Delta$ is as defined in Theorem \ref{theo:correlated}), we have $|\widehat{x}_i| = |\widehat{x}_j|$.
\end{theorem}

Under the normalization assumptions on matrix $\bm A$ that were used in Corollary \ref{cor:normalized}, another (weaker) sufficient condition can be obtained which depends on the sample correlations, as stated in the following corollary.

\begin{corollary}
Let $\widehat{\bm x}$ be any minimizer of the objective function in \eqref{OWL_L1} and assume that the columns of $\bm A$ are normalized, that is, ${\bm 1}^T {\bm a}_k = 0$ and $\| {\bm a}_k\|_2 = 1$, for $i=k,...,p$. Denote their inner products ({\it i.e.}, the sample correlation of the corresponding explanatory variables) as $\rho_{ij} = {\bm a}_i^T {\bm a}_j / (\|{\bm a}_i\|_2 \, \|{\bm a}_j\|_2)= {\bm a}_i^T {\bm a}_j$.  Then, for every pair of columns $(i,j)$ for which $ \sqrt{n (2 - 2 \, \rho_{ij}\,\sign(\widehat{x}_i\,\widehat{x}_j))} < \Delta$, we have $|\widehat{x}_i| = |\widehat{x}_j|$.
\end{corollary}

\begin{proof}
The corollary results simply from noticing that, under the assumed normalization of the columns of ${\bm A}$, $\| \sign(\widehat{x}_i)\, {\bm a}_i - \sign(\widehat{x}_j)\,{\bm a}_j \|_1 \leq \sqrt{n}\, \|\sign(\widehat{x}_i)\,{\bm a}_i - \sign(\widehat{x}_j)\,{\bm a}_j\|_2  = \sqrt{n\, (2-2\,\rho_{ij}\,\sign(\widehat{x}_i\,\widehat{x}_j))}.$
\end{proof}

Finally, a simple corollary results from the fact that, for OSCAR, $\Delta = \lambda_2$.

\begin{corollary}
In the particular case of OSCAR, the condition is $ \|\sign(\widehat{x}_i)\,{\bm a}_i - \sign(\widehat{x}_j)\,{\bm a}_j \|_1 < \lambda_2 $, in general, and $\sqrt{n (2 - 2 \, \rho_{ij}\,\sign(\widehat{x}_i\,\widehat{x}_j))} < \lambda_2$, in the case of normalized columns.
\end{corollary}

Finally, as above, in the extreme case of perfectly correlated columns ($\rho_{ij}\, \sign(\widehat{x}_i\,\widehat{x}_j) = 1$), the condition for OSCAR simplifies to $\lambda_2 > 0$.

\subsection{Proofs of Theorems \ref{theo:correlated} and \ref{theo:correlatedL1}}

The proofs of both Theorems \ref{theo:correlated} and \ref{theo:correlatedL1} are based on a useful lemma about the OWL norm, which we state and prove before proceeding to the proofs of the theorems.

\begin{lemma} \label{lem:delta} Consider a vector ${\bm x} \in \mathbb{R}_+^p$ and any two of its components $x_i$ and $x_j$, such that $x_i > x_j$. Let ${\bm z} \in \mathbb{R}_+^p$ be obtained by applying a so-called Pigou-Dalton\footnote{The Pigou-Dalton transfer, also known as a Robin Hood transfer, is used in the study of measures of economic inequality \cite{Dalton}, \cite{Pigou}.} transfer of size $\varepsilon \in \bigl( 0,\, (x_i - x_j)/2\bigr)$ to $\bm x$, that is: $z_i = x_i - \varepsilon$, $z_j = x_j + \varepsilon$, and $z_k = x_k,$ for $k\neq i,j$. Let $\bm w$ be a vector of non-increasing non-negative real values, $w_1 \geq w_2 \geq \cdots \geq w_p \geq 0$, and $\Delta$ be the minimum gap between two consecutive components of vector $\bm w$, that is, $\Delta = \min \{ w_l - w_{l+1},\; l=1,...,p-1 \}$. Then,
\begin{equation}
\Omega_{\bm w} ({\bm x}) - \Omega_{\bm w}({\bm z}) \geq \Delta \, \varepsilon .\label{eq:delta1}
\end{equation}
\end{lemma}

\begin{proof}
Let $l$ and $m$ be the rank orders of $x_i$ and $x_j$, respectively, {\it i.e.}, $x_i = x_{[l]}$ and $x_j = x_{[m]}$; of course, $m>l$, because $x_i > x_j$.
Now let $l+a$ and $m-b$ be the rank orders of $z_i$ and  $z_j$, respectively, {\it i.e.}, $x_i - \varepsilon = z_i = z_{[l+a]}$ and $x_j + \varepsilon = z_j = z_{[m-b]}$. Of course, it may happen that $a$ or $b$ (or both) are zero, if $\varepsilon$ is small enough not to change the rank orders of one (or both) of the affected components of $\bm x$. Furthermore, the condition $\varepsilon < (x_i - x_j)/2$ implies that $x_i - \varepsilon > x_j + \varepsilon$, thus $l+a < m-b$. A key observation is that $x_{\downarrow}$ and $z_{\downarrow}$ only differ in positions $l$ to $l+a$ and $m-b$ to $m$, thus we can thus write
\begin{eqnarray}
\Omega_{\bm w} ({\bm x}) - \Omega_{\bm w}({\bm z}) = \sum_{k=l}^{l+a} w_k \, \bigl(x_{[k]} - z_{[k]}\bigr)
+ \sum_{k=m-b}^{m} w_k \, \bigl(x_{[k]} - z_{[k]}\bigr).\label{eq:lemma1}
\end{eqnarray}
In the range from $l$ to $l+a$, the relationship between $\bm z_{\downarrow}$ and $\bm x_{\downarrow}$ is
\begin{equation}
z_{[l]} = x_{[l+1]},  \; z_{[l+1]} = x_{[l+2]},\; \dots\; ,\; z_{[l+a-1]} = x_{[l+a]},\; z_{[l+a]} = x_{[l]} - \varepsilon,
\end{equation}
whereas in the range from $m-b$ to $m$, we have
\begin{equation}
z_{[m-b]} = x_{[m]} + \varepsilon,  \; z_{[m-b+1]} = x_{[m-b]},\; \dots \; ,\; z_{[m]} = x_{[m-1]}.
\end{equation}
Plugging these equalities into \eqref{eq:lemma1} yields
\begin{eqnarray}
\Omega_{\bm w} ({\bm x}) - \Omega_{\bm w}({\bm z}) & = &  \sum_{k=l}^{l+a-1} w_k \, \bigl( x_{[k]} - x_{[k+1]} \bigr) + \sum_{k=m-b+1}^{m} w_k \,  \bigl( x_{[k]} - x_{[k-1]}\bigr) \nonumber\\
  & & + \, w_{l+a}  \bigl(x_{[l+a]} - x_{[l]} + \varepsilon\bigr) + w_{m-b}  \bigl(x_{[m-b]} - x_{[m]} - \varepsilon\bigr) \nonumber \\
 & \stackrel{(a)}{\geq} &  w_{l+a} \sum_{k=l}^{l+a-1}  \, \bigl( x_{[k]} - x_{[k+1]} \bigr) + w_{m-b} \sum_{k=m-b+1}^{m}   \bigl( x_{[k]} - x_{[k-1]}\bigr) \nonumber\\
& & + \, w_{l+a}  \bigl(x_{[l+a]} - x_{[l]} + \varepsilon\bigr) + w_{m-b}  \bigl(x_{[m-b]} - x_{[m]} - \varepsilon\bigr) \nonumber\\
& = &  w_{l+a} \left( \sum_{k=l}^{l+a-1}  \, \bigl( x_{[k]} - x_{[k+1]} \bigr) +
\bigl(x_{[l+a]} - x_{[l]} + \varepsilon\bigr) \right) \nonumber\\
& & + w_{m-b} \left( \sum_{k=m-b+1}^{m}   \bigl( x_{[k]} - x_{[k-1]}\bigr) + \bigl(x_{[m-b]} - x_{[m]} - \varepsilon\bigr)\right)\nonumber\\
& = & \varepsilon \, \bigl( w_{l+a} - w_{m-b}\bigr) \nonumber \\
& \geq & \varepsilon \, \Delta\nonumber,
\end{eqnarray}
where inequality $(a)$ results from $x_{[k]} - x_{[k+1]} \geq 0$, $x_{[k]} - x_{[k-1]} \leq 0$, and the components of $\bm w$ forming a non-increasing sequence.
\end{proof}

Armed with Lemma \ref{lem:delta}, we now proceed to prove Theorems  \ref{theo:correlated} and \ref{theo:correlatedL1}.

\begin{proof} (Theorem \ref{theo:correlated})
Let us denote $L(\bm x) = \tfrac{1}{2}\| \bm A\, \bm x - \bm y\|_2^2$ and take some pair of columns $(i,j)$.  Suppose that $\widehat{\bm x}$ is a minimizer of the objective function in \eqref{OWL_L2}, satisfying the condition of the theorem ($\Delta > \|\bm y\|_2\; \|\sign(\widehat{x}_i)\,{\bm a}_i - \sign(\widehat{x}_j)\,{\bm a}_j \|_2$), but in contradiction to the theorem's claim, {\it i.e.}, for which $|\widehat{x}_i| \neq |\widehat{x}_j|$. Without loss of generality, assume that $|\widehat{x}_i| > |\widehat{x}_j|$, and define the residual vector
\begin{equation}\label{eq:residual}
\bm g = \bm y \;\;- \sum_{k=1, \, k\neq i,\, k\neq j}^p \widehat{x}_k\; {\bm a}_k.
\end{equation}
Now consider a Pigou-Dalton transfer of size $\varepsilon < \min\{|\widehat{x}_i|, (|\widehat{x}_i| - |\widehat{x}_j|)/2\}$ applied to the magnitudes of $\widehat{x}_i$ and $\widehat{x}_j$, {\it i.e.}, take an alternative candidate solution $\bm v\in \mathbb{R}^p$, such that, $v_i = \sign(\widehat{x}_i) (|\widehat{x}_i| - \varepsilon)$, $v_j = \sign(\widehat{x}_i) (|\widehat{x}_i| + \varepsilon)$, and $v_k = \widehat{x}_k$, for $k\neq i,j$. Denoting $\widetilde{\bm a}_i = \sign(\widehat{x}_i)\, \bm a_i$ and $\widetilde{\bm a}_j = \sign(\widehat{x}_j)\, \bm a_j$, the diference in loss function that results from this transfer is
\begin{eqnarray}
L(\bm v) - L(\widehat{\bm x}) & = & \frac{1}{2}\,\bigl\| \bm g - (|\widehat{x}_i| - \varepsilon)\, \widetilde{\bm a}_i -  (|\widehat{x}_j| + \varepsilon)\, \widetilde{\bm a}_j \bigr\|_2^2 - \frac{1}{2}\, \bigl\| \bm g - |\widehat{x}_i| \, \widetilde{\bm a}_i -  |\widehat{x}_j| \, \widetilde{\bm a}_j \bigr\|_2^2.
\end{eqnarray}
Expanding the squared $\ell_2$ norms, cancelling out the common $\tfrac{1}{2}\|\bm g\|_2^2$ term, and using the common norm of the columns ($\|{\bm a}_k\|_2 = c$, for $k=1,...,p$) leads to
\begin{eqnarray}
L(\bm v) - L(\widehat{\bm x}) & = & \frac{1}{2} (|\widehat{x}_i| - \varepsilon)^2 c^2 + \frac{1}{2} (|\widehat{x}_j| + \varepsilon)^2 c^2 - (|\widehat{x}_i| - \varepsilon)\, {\bm g}^T \widetilde{\bm a}_i - (|\widehat{x}_j| + \varepsilon)\, {\bm g}^T \widetilde{\bm a}_j \nonumber\\
& & + (|\widehat{x}_i| - \varepsilon)\,(|\widehat{x}_j| + \varepsilon)\,  \widetilde{\bm a}_i^T \widetilde{\bm a}_j - \frac{1}{2} |\widehat{x}_i|^2 c^2 - \frac{1}{2} |\widehat{x}_j|^2 c^2 \nonumber\\
& & + |\widehat{x}_i| \, {\bm g}^T \widetilde{\bm a}_i + |\widehat{x}_j| \, {\bm g}^T \widetilde{\bm a}_j - |\widehat{x}_i|\, |\widehat{x}_j| \, \widetilde{\bm a}_i^T \widetilde{\bm a}_j.
\end{eqnarray}
Expanding the terms $(|\widehat{x}_i| - \varepsilon)^2$,
$(|\widehat{x}_j| + \varepsilon)^2$, and $(|\widehat{x}_i| - \varepsilon)\,(|\widehat{x}_j| + \varepsilon)$ and making some further cancellations yields
\begin{eqnarray}
L(\bm v) - L(\widehat{\bm x}) & = & \varepsilon \, {\bm g}^T\bigl( \widetilde{\bm a}_i - \widetilde{\bm a}_j\bigr) +  \varepsilon^2 (c^2 -  \widetilde{\bm a}_i^T \widetilde{\bm a}_j )  - \varepsilon \, c^2\, \bigl(|\widehat{x}_i| - |\widehat{x}_j|\bigr) + \varepsilon (|\widehat{x}_i|-|\widehat{x}_j|) \, \widetilde{\bm a}_i^T \widetilde{\bm a}_j\nonumber\\
& = & \varepsilon \, {\bm g}^T\bigl( \widetilde{\bm a}_i - \widetilde{\bm a}_j\bigr)
 - \varepsilon (c^2 -  \widetilde{\bm a}_i^T \widetilde{\bm a}_j ) \bigl( |\widehat{x}_i|-|\widehat{x}_j| - \varepsilon \bigr) \nonumber \\
& \stackrel{(a)}{\leq} & \varepsilon\; {\bm g}^T (\widetilde{\bm a}_i - \widetilde{\bm a}_j) \nonumber\\
& \stackrel{(b)}{\leq} & \varepsilon\; \|{\bm y}\|_2\; \| \widetilde{\bm a}_i - \widetilde{\bm a}_j\|_2, \nonumber
\end{eqnarray}
where inequality $(a)$ results from the facts that (by the Cauchy-Schwartz inequality) $c^2 \geq \widetilde{\bm a}_i^T \widetilde{\bm a}_j$
and both $\varepsilon$ and $\bigl( |\widehat{x}_i|-|\widehat{x}_j| - \varepsilon \bigr)$ are (by assumption) positive, whereas $(b)$ is again Cauchy-Schwartz together with the fact that $\|\bm g\|_2 \leq \|\bm y\|_2$.
Finally, since $|\bm v| \in \mathbb{R}_+^p$ results from the same Pigou-Dalton transfer  of size $\varepsilon$ applied to $|\widehat{\bm x}|\in \mathbb{R}_+^p$, and $\Omega_{\bm w}$ only depends on the absolute values of its arguments, we are in condition to invoke Lemma \ref{lem:delta}, which yields
\begin{equation}
L(\bm v) + \Omega_{\bm w} (\bm v) - (L(\widehat{\bm x}) + \Omega_{\bm w}(\widehat{\bm x})) \leq \varepsilon\; \bigl(\|\bm y\|_2\; \|\widetilde{\bm a}_i - \widetilde{\bm a}_j\|_2 - \Delta\bigr) < 0,
\end{equation}
contradicting the assumption that $\widehat{\bm x}$ is a minimizer of $L(\bm x) + \Omega_{\bm w}(\bm x)$, thus completing the proof.
\end{proof}

\begin{proof} (Theorem \ref{theo:correlatedL1})
Denote $G(\bm x) = \| \bm A\, \bm x - \bm y\|_1$ and take some pair of columns $(i,j)$. Assume that $\widehat{\bm x}$ is a minimizer of the objective function in \eqref{OWL_L1}, satisfying the condition of the theorem ($\Delta > \|\bm \sign(\widehat{x}_i)\,{\bm a}_i - \sign(\widehat{x}_j)\,\bm a_j\|_1$, but in
contradiction to the theorem's claim, {\it i.e.}, for which $|\widehat{x}_i| \neq |\widehat{x}_j|$. Define the residual vector $\bm g$ as in \eqref{eq:residual} and
consider a Pigou-Dalton transfer of size $\varepsilon < \min\{|\widehat{x}_i|, (|\widehat{x}_i| - |\widehat{x}_j|)/2\}$ applied to the magnitudes of $\widehat{x}_i$ and $\widehat{x}_j$, {\it i.e.}, take an alternative candidate solution $\bm v\in \mathbb{R}^p$, such that, $v_i = \sign(\widehat{x}_i) (|\widehat{x}_i| - \varepsilon)$, $v_j = \sign(\widehat{x}_i) (|\widehat{x}_i| + \varepsilon)$, and $v_k = \widehat{x}_k$, for $k\neq i,j$. Denoting $\widetilde{\bm a}_i = \sign(\widehat{x}_i)\, \bm a_i$ and $\widetilde{\bm a}_j = \sign(\widehat{x}_j)\, \bm a_j$, the diference in loss function that results from this transfer satisfies
\begin{eqnarray}
G(\bm v) - G(\widehat{\bm x})  & = & \bigl\|\bm g - (|\widehat{x}_i| - \varepsilon) \widetilde{\bm a}_i - (|\widehat{x}_i| + \varepsilon) \widetilde{\bm a}_i\bigr\|_1 - \bigl\|\bm g - |\widehat{x}_i| \, \widetilde{\bm a}_i - |\widehat{x}_j| \, \widetilde{\bm a}_j\bigr\|_1 \nonumber\\
& = & \bigl\|\bm g - |\widehat{x}_i| \widetilde{\bm a}_i - |\widehat{x}_i| \widetilde{\bm a}_i + \varepsilon (\widetilde{\bm a}_i - \widetilde{\bm a}_j)\bigr\|_1 - \bigl\|\bm g - |\widehat{x}_i| \, \widetilde{\bm a}_i - |\widehat{x}_j| \, \widetilde{\bm a}_j\bigr\|_1 \nonumber\\
& \leq &  \varepsilon \bigl\|\widetilde{\bm a}_i - \widetilde{\bm a}_j\bigr\|_1 , \label{delta_L3}
\end{eqnarray}
as a direct consequence of the triangle inequality.
Finally, since $|\bm v| \in \mathbb{R}_+^p$ results from the same Pigou-Dalton transfer  of size $\varepsilon$ applied to $|\widehat{\bm x}|\in \mathbb{R}_+^p$, and $\Omega_{\bm w}$ only depends on the absolute values of its arguments, we are in condition to invoke Lemma \ref{lem:delta}, thus
\begin{equation}
G(\bm v) + \Omega_{\bm w}(\bm v) - (G(\widehat{\bm x}) + \Omega_{\bm w}(\widehat{\bm x})) \leq \varepsilon\; \bigl( \| \widetilde{\bm a}_i - \widetilde{\bm a}_j \|_1 - \Delta\bigr) < 0,
\end{equation}
which contradicts the assumption that $\widehat{\bm x}$ is a minimizer of $G(\bm x) + \Omega_{\bm w}(\bm x)$, thus completing the proof.
\end{proof}

\subsection{Proof of Theorem \ref{thm_preview} (i)}
\label{pi}
The proof of item (i) in Theorem \ref{thm_preview} follows the same general structure as the proofs of Theorems \ref{theo:correlated} and \ref{theo:correlatedL1}.

\begin{proof} (Theorem \ref{thm_preview} (i))
Let us define the functions  $L(\bm x) = \tfrac{1}{n}\| \bm A\, \bm x - \bm y\|_2^2$ and $G(\bm x) = \tfrac{1}{n}\| \bm A\, \bm x - \bm y\|_1$, and the residual $\bm g$ as in \eqref{eq:residual}. Since ${\bm a}_i = {\bm a}_j$, the functions
\begin{equation}
L(\widehat{\bm x}) = \frac{1}{n}\bigl\| {\bm g} - (\widehat{x}_i + \widehat{x}_j){\bm a}_i\bigr\|_2^2 \hspace{1cm}\mbox{and}\hspace{1cm}
G(\widehat{\bm x}) = \frac{1}{n}\bigl\| {\bm g} - (\widehat{x}_i + \widehat{x}_j){\bm a}_i\bigr\|_1 ,\label{eq:equal_cols}
\end{equation}
are both invariant under a transformation that adds any quantity $\varepsilon$ to $\widehat{x}_i$ and subtracts the same quantity from $\widehat{x}_j$.

We first prove that if ${\bm a}_i = {\bm a}_j$, then $\sign({\widehat{x}_i}) = \sign({\widehat{x}_j})$. Assume, by contradiction, that $\sign({\widehat{x}}_i) \neq \sign(\widehat{x}_j)$, and, without loss of generality,  that $\widehat{x}_i > 0$. We need to consider two cases:
\begin{description}
\item[a)] if $\widehat{x}_j < 0$, take an alternative feasible solution ${\bm v}$, with: $v_k = \widehat{x}_k$, for $k \neq i,j$,  $v_i = \widehat{x}_i - \varepsilon$, and $v_j = \widehat{x}_j + \varepsilon$, for some $\varepsilon \in (0,\, \min\{|\widehat{x}_i|,|\widehat{x}_j|\}] $. Since $\widehat{x}_i > 0$ and $\widehat{x}_j < 0$, it's true that $|v_i| = |\widehat{x}_i| - \varepsilon$ and $|v_j| = |\widehat{x}_j| - \varepsilon$. Finally, the definition of $\Delta$ implies that $w_{p-1} \geq \Delta > 0$, thus $\Omega_{\bm w}(\widehat{\bm x}) - \Omega_{\bm w}(\bm v) > \varepsilon \Delta > 0$, contradicting the optimality of $\widehat{\bx}$, thus proving the claim that $\sign({\widehat{x}_i}) = \sign({\widehat{x}_j})$.
\item[b)] if $\widehat{x}_j = 0$, consider a feasible ${\bm v}$ resulting from a Pigou-Dalton transfer of size $\varepsilon \in \bigl( 0,\, |\widehat{x}_i| / 2\bigr] $. From Lemma \ref{lem:delta},  $\Omega_{\bm w}(\widehat{\bm x}) - \Omega_{\bm w}(\bm v) > \varepsilon \Delta > 0$, negating the optimality of $\widehat{\bx}$, thus proving that $\sign({\widehat{x}_i}) = \sign({\widehat{x}_j})$.
\end{description}

Once it is established that $\sign({\widehat{x}_i}) = \sign({\widehat{x}_j})$, we proceed to prove that $|{\widehat{x}_i}|=|{\widehat{x}_j}|$. To this end, notice that
\begin{equation}
L(\widehat{\bm x}) = \frac{1}{n}\bigl\| {\bm g} - (|\widehat{x}_i| + |\widehat{x}_j|)\, \sign({\widehat{x}_i})\, {\bm a}_i\bigr\|_2^2 \hspace{1cm}\mbox{and}\hspace{1cm}
G(\widehat{\bm x}) = \frac{1}{n}\bigl\| {\bm g} - (|\widehat{x}_i| + |\widehat{x}_j|)\, \sign({\widehat{x}_i})\, {\bm a}_i\bigr\|_1 .\label{eq:equal_cols}
\end{equation}
Proceeding again by contradiction, suppose (without loss of generality) that $|{\widehat{x}_i}| > |{\widehat{x}_j}|$, and define $\bm u$ via a Pigou-Dalton transfer on the magnitudes of $\bm x_i$ and $\bm x_j$, that is: $u_k = \widehat{x}_k$, for $k \neq i,j$,  $u_i = (|\widehat{x}_i| - \delta)\, \sign({\widehat{x}_i})$, and $u_j = (|\widehat{x}_j| + \delta)\, \sign({\widehat{x}_i})$, for some $\delta \in (0,\, \min\{|\widehat{x}_i|,(|\widehat{x}_i|-|\widehat{x}_j|)/2\}]$. Of course, $\bm u$ is feasible and Lemma \ref{lem:delta} shows that $\Omega_{\bm w}(\widehat{\bm x}) - \Omega_{\bm w}(\bm u) > \delta \Delta > 0$, contradicting the optimality of $\widehat{\bx}$, thus concluding the proof.
\end{proof}

\section{Statistical Analysis of OWL Regularization}
\label{sec:stat}
\subsection{Introduction}
In this section, we characterize the statistical performance of OWL regularization with both the squared and absolute error losses, by proving finite sample bounds, which apply to the standard LASSO and OSCAR as special cases. At the basis of our approach is the following model for correlated measurement matrices.  Recall that $\bA$ has dimensions $n\times p$. Assume that the rows of $\bA$ are independently and identically distributed ${\cal N}({\bf 0},\bC^T\bC)$, the multivariate Gaussian distribution with covariance $\bC^T\bC$ ({\it i.e.}, the columns of $\bA$ are not necessarily independent).  Assume that the matrix $\bC$ is $q\times p$ with $q\geq n$.  Note that $\bA$ can be factorized as $\bA = \bB\bC$, where $\bB$ is an $n \times q$ Gaussian random matrix, whose entires are i.i.d.\ ${\cal N}(0,1)$ random variables.  The role of matrix $\bC$ is to mix, or even replicate, columns of $\bB$. The next simple example illustrates this construction.

\begin{example}
Suppose $q=3$, $p=4$, and
\begin{equation}
\bC = \left[\begin{array}{cccc} 1 & 1 & 0 & 0 \\ 0 & 0 & 1 & 0 \\ 0 & 0 & 0 & 1\end{array}\right] ;
\label{c}
\end{equation}
then, if $\bB = [{\bm b}_1,\, {\bm b}_2, \, {\bm b}_3]$, where $\bm b_i \in \mathbb{R}^n$ is the $i$-th column of $\bB$, matrix $\bA$ has the form $\bA = [{\bm b}_1,\, {\bm b}_1,\, {\bm b}_2, \, {\bm b}_3]$.
\label{ex1}
\end{example}

Assume that we observe
\begin{eqnarray}
\by & = & \bA\, \bx^\star \ + \ \bnu ,
\label{obs}
\end{eqnarray}
where $\bnu \in \R^n$ is the measurement error satisfying
\begin{eqnarray}
\frac{1}{n} \|\bnu\|_1 \ \leq \ \varepsilon ,
\end{eqnarray}
and about which we make no other assumptions.  The signal $\bx^\star\in \R^p$ is assumed to satisfy $\|\bx^\star\|_1\leq \sqrt{s}$.  Note, for example, that this condition is met if $\|\bx^\star\|_2 \leq 1$ and $\bx^\star$  has  at most $s$ non-zero components.

\begin{example}
To illustrate the model, consider the matrix $\bC$ defined in (\ref{c}) and suppose that $\bA\bx^\star = \bb_1$. There are many $\bx \in \R^4$ satisfying
$\bA\bx = \bb_1$, including the vectors $[1,\, 0, \, 0, \, 0]^T$, $[0, \, 1, \, 0, \, 0]^T$, and all convex combinations of the two. LASSO regularization ({\it i.e.}, the $\ell_1$ norm) does not differentiate these equivalent representations, since $\|[\alpha,\, (1-\alpha), \, 0, \, 0]^T\|_1 = 1$, for any $\alpha \in [0,\, 1]$. However the OWL norm (as claimed in Theorem \ref{thm_preview}) prefers the solution $\left[\frac{1}{2}, \, \frac{1}{2}, \, 0, \, 0 \right]^T$, selecting both colinear columns of $\bA$ for the representation.
\label{ex2}
\end{example}

\subsection{Main Result and Corollaries}\label{sec:main_and_corollaries}
The main result of this section is stated in the following theorem, whose proof is given in the following subsections, based on the techniques developed by Vershynin \cite{vershynin14}.  The results are stated in terms of constrained optimization,
which, under certain conditions, is equivalent to the Lagrangian formulation studied in Section \ref{sec:sufficient_clustering}. We also present a corollary for the particular case where $\bC$ simply replicates columns of $\bB$, that is, when $\bA$ includes groups of identical columns; this corollary is shown to imply part (ii) of Theorem \ref{thm_preview}. {Expectations are with respect to the Gaussian distribution of $\bA$.}

\begin{theorem}
Let $\by$, $\bA$, $\bx^\star$, and $\varepsilon$ be as defined above, and let $\widehat{\bx}$ be a solution to one of the two following optimization problems:
\begin{eqnarray}
\min_{\bx \in \R^p} \Omega_{\bm w}(\bx) \mbox{ \ subject to \ } \frac{1}{n}\|\bA\bx-\by\|_2^2 \ \leq \ \varepsilon^2 ,
\end{eqnarray}
or
\begin{eqnarray}
\min_{\bx \in \R^p} \Omega_{\bm w}(\bx)  \mbox{ \ subject to \ } \frac{1}{n}\|\bA\bx-\by\|_1  \ \leq \  \varepsilon  .
\end{eqnarray}
Then
\begin{eqnarray}
\E \sqrt{(\widehat \bx - \bx^\star)^T\bC^T\bC(\widehat \bx - \bx^\star )} & \leq & \sqrt{2\pi} \left(  4\sqrt{2}\, \|\bC\|_1 \, \frac{w_1}{\bar w} \, \sqrt{\frac{s\log q}{n}}  +\varepsilon\right) ,
\label{ebound}
\end{eqnarray}
where $\bar w = p^{-1} \sum_{i=1}^p w_i$ and $ \|\bC\|_1$ is the matrix norm induced by the $\ell_1$ norm: $ \|\bC\|_1 = \max_j \|{\bm c}_j\|_1$, with $\bm c_j$ denoting the $j$-th column of $\bm C$.
\label{thm1}
\end{theorem}
\sloppypar Note that the error bound in (\ref{ebound})  holds for optimizations based on the squared $\ell_2$ and $\ell_1$ losses. In fact, since \mbox{$\frac{1}{n}\|\bA\bx-\by\|_2^2$} $\leq \varepsilon^2$ implies $\frac{1}{n}\|\bA\bx-\by\|_1  \leq  \varepsilon$, the $\ell_1$ constraint is less restrictive.
In both cases, the theorem shows that the number of samples sufficient to estimate an $s$-sparse signal with a given precision grows like
$$n \ \sim \ s\log q \ .  $$
This agrees with well-known sample complexity bounds for sparse recovery under stronger assumptions such as the restricted isometry property or i.i.d.\ measurements  \cite{candes06,donoho06,haupt06,candes07,vershynin14}. However, note that the error measure of the theorem is insensitive to components of $\widehat \bx$ in the nullspace of $\bC$, which is to be expected since in general there may be many sparse $\bx$ that yield the same value of $\bA\bx$ (see Example~\ref{ex2}).  This is where the OWL norm becomes especially important.  With strictly decreasing weights ($\Delta > 0$), OWL prefers solutions that select all colinear columns associated with the model.  In other words, if the columns are colinear (or strongly correlated, per the characterizations of given in Section \ref{sec:sufficient_clustering}), then the OWL solution will select a representation including all the columns associated with the sparse model, rather than an arbitrary subset of them.

The OSCAR norm is a special case of OWL, with $w_i = \lambda_1+\lambda_2(p-i)$ and  $\lambda_1,\lambda_2>0$. In this case, $\bar w = \lambda_1+\lambda_2 (p-1)/2$ and therefore $w_1 /\bar w \leq 2$. Note that the conventional $\ell_1$ norm (used in the LASSO) is the special case of OWL with uniform weights (or OSCAR with $\lambda_2 = 0$); thus, all our results apply to $\ell_1$ minimization as well, in which case $w_1 /\bar w =1$.

To illustrate Theorem \ref{thm1}, let $\bC$ be an $q\times p$ matrix that replicates each column of $\bB$ one or more times.  Note that each column of $\bC$ is $1$-sparse and has unit $\ell_1$ norm, thus $\|\bC\|_1 = 1$.  Let $G_1,\dots,G_q$ denote the
groups of replicated columns in $\bA = \bB\bC$; these groups are a partition of the set $\{1,\dots,p\}$.  Example~\ref{ex1} is a special case of this scenario, with $G_1 = \{1,2\}$,
$G_2 = \{3\}$, and $G_3 = \{4\}$.
Assume that there are $s$ non-zero components in $\bx^\star$, each in one of $s$ distinct groups. Let $\bx_G$ denote the vector that is zero except on the the subset of entries in $G\subset \{1,\dots,p\}$, where it takes the same values as $\bx$.  Then note that for any $\widehat\bx$ we have
\begin{eqnarray}
(\widehat \bx - \bx^\star)^T\bC^T\bC(\widehat \bx - \bx^\star) \ = \ \sum_{i=1}^q |{\bf 1}^T (\widehat{\bx}_{G_i} -\bx_{G_i}^\star)|^2 \ ,
\label{sp}
\end{eqnarray}
where ${\bf 1} = [1 \, 1 \, \dots \, 1]^T$.
This produces the following corollary to Theorem~\ref{thm1}.
\begin{corollary}
Assume that each column of $\bC$ is $1$-sparse and unit norm.  Let $\widehat{\bx}$ be a solution to the optimization
\begin{eqnarray}
\min_{\bx \in \R^p} \Omega_{\bm w}(\bx) \mbox{ \ subject to \ } \frac{1}{n}\|\bA\bx-\by\|_2^2 \ \leq \ \varepsilon^2 ,
\end{eqnarray}
or
\begin{eqnarray}
\min_{\bx \in \R^p} \Omega_{\bm w}(\bx)  \mbox{ \ subject to \ } \frac{1}{n}\|\bA\bx-\by\|_1  \ \leq \  \varepsilon .
\end{eqnarray}
Then
\begin{eqnarray}
\E \sqrt{\sum_{i=1}^q |{\bf 1}^T (\widehat{\bx}_{G_i} -\bx_{G_i}^\star)|^2} & \leq & \sqrt{2\pi} \left(  4\sqrt{2} \, \frac{w_1}{\bar w} \, \sqrt{\frac{s \log q}{n}}  +\varepsilon\right) .
\label{e2}
\end{eqnarray}
\label{cor1}
\end{corollary}

In this case, since the correlated columns are colinear, the OWL norm will select all or none of the columns in each group and each $\widehat \bx_{G_i}$ will have identical non-zero values (if any). If we let $z_i^\star = {\bf 1}^T \bx_{G_i}^\star$ and $\widehat z_i = {\bf 1}^T \widehat \bx_{G_i}$ for $i=1,\dots,q$, and let  $\bz^\star = [z_1^\star,...,z_q^\star]^T$ and $\widehat \bz = [\widehat z_1,...,\widehat z_q]^T$; then (\ref{e1}) can be expressed as
\begin{eqnarray}
\E \|\widehat \bz -\bz^\star\|_2 & \leq & \sqrt{2\pi} \left(  4\sqrt{2} \, \frac{w_1}{\bar w} \, \sqrt{\frac{s\log q}{n}}  +\varepsilon\right) \ ,
\label{zp}
\end{eqnarray}
which is the type of result obtained in the compressed sensing literature for the ideal
sparse observation model $\by = \bB\bz^\star+\bnu$ (based on an i.i.d.\ observation model) \cite{vershynin14}. This shows that {\em by using OWL we pay no price for colinearity in} $\bA$.
Also, as shown next, Corollary \ref{cor1} implies claim (ii) in Theorem~\ref{thm_preview} in Section 1.

\begin{proof} (Theorem \ref{thm_preview} (ii))
Taking into account the group structure of $\widehat \bx$ and $\bx^\star$, we have that
\[
\|\widehat \bx - \bx^\star\|_2 = \sqrt{ \sum_{i=1}^q \|\widehat \bx_{G_i} - \bx_{G_i}^\star\|_2^2}
= \sqrt{ \sum_{i=1}^q \frac{1}{|G_i|} (\widehat z_i - z_i^\star)^2  } \leq
 \|\widehat \bz -\bz^\star\|_2;
\]
this inequality, together with \eqref{zp} and the fact that $q \leq p$ yields \eqref{e1}.
\end{proof}

Before moving on to the proof of Theorem~\ref{thm1}, notice that since the $\ell_1$ norm is a special case of OWL with uniform weights,  the same bounds in Theorem~\ref{thm1} and Corollary~\ref{cor1} hold for $\ell_1$ minimization.  The difference is that the LASSO solution generally will not select all correlated or colinear columns selected by OWL, making the estimated model less interpretable.

\subsection{Proof of Theorem~\ref{thm1}}
The proof of Theorem~\ref{thm1} is based on the approach developed by Vershynin \cite{vershynin14}.
The key ingredient is the so-called {\it general $M^*$ bound} (Theorem 5.1 in \cite{vershynin14}), which applies to the special case when $\bA$ is a i.i.d. Gaussian matrix (i.e., when $\bC$ is identity in our set-up).  We extend the bound to cover our model $\bA = \bB\bC$, for general $\bC$.  Recall that $\bA$ is $n\times p$, $\bB$ is $n\times q$, with $n \leq q$, and $\bC$ is $q\times p$.

\subsubsection{Extended General $M^*$ Bound}

\begin{theorem}
\rm{(Extended general $M^*$ bound).} {\em Let $T$ be a bounded subset of $\R^p$.  Let $\bB$ be an $n\times q$ Gaussian random matrix (with i.i.d.\ ${\cal N}(0,1)$ entries) and let $\bA = \bB\bC$, where $\bC$ is a deterministic $q\times p$ matrix.  Fix $\varepsilon \geq 0$ and consider the set
\begin{eqnarray}
T_\varepsilon & := & \left\{\bu \in T \ : \ \frac{1}{n} \|\bA\bu\|_1 \leq \varepsilon \right\} \ .
\label{te}
\end{eqnarray}
Then
\begin{eqnarray}
\E \sup_{\bu \in T_\varepsilon} \left(\bu^T\bC^T\bC\bu\right)^{1/2} & \leq & \sqrt{\frac{2\pi}{n}} \, \E \sup_{\bu \in T} |\langle \bC^T\bg,\bu\rangle| \ + \ \sqrt{\frac{\pi}{2}} \varepsilon \ ,
\label{mbound}
\end{eqnarray}
where $\bg \sim {\cal N}(0,{\bf I}_q)$ is a standard Gaussian random vector in $\R^q$.
\label{egm}}
\end{theorem}
The proof follows in a straightforward fashion from the proof of Theorem 5.1 in \cite{vershynin14}, with modifications made to account for $\bC$.  For the sake of completeness we include a proof here.
\begin{proof}
The bound (\ref{mbound}) follows from the deviation inequality
\begin{eqnarray}
\E \sup_{\bu \in T} \left|\frac{1}{n}\sum_{i=1}^n |\langle \ba_i,\bu \rangle| - \sqrt{\frac{2}{\pi}} \left(\bu^T\bC^T\bC\bu\right)^{1/2} \right| \ \leq \ \frac{2}{\sqrt{n}} \, \E \sup_{\bu \in T} |\langle \bC^T\bg,\bu\rangle| \ ,
\label{bb}
\end{eqnarray}
where $\ba_i$ denotes the $i$th row of $\bA$.  To see this, note that the inequality holds if we replace $T$ by the smaller set $T_{\varepsilon}$. For $\bu \in T_{\varepsilon}$,  and for such $\bu$ we have by assumption that $\frac{1}{n}\sum_{i=1}^n |\langle \ba_i,\bu \rangle| = \frac{1}{n}\|\bA\bu\|_1 \leq \varepsilon$, and the bound (\ref{mbound}) follows by the triangle inequality.

To prove (\ref{bb}), the first thing to note is that
\begin{eqnarray}
\E |\langle \ba_i,\bu \rangle| & = & \E |\langle \bC^T\bb_i,\bu \rangle| \ = \ \E |\langle \bb_i,\bC\bu \rangle| \ ,
\end{eqnarray}
where $\bb_i$ is the $i$th row of $\bB$.  Because the Gaussian distribution of $\bb_i$ is rotationally invariant, it follows that
\begin{eqnarray}
\E |\langle \bb_i,\bC\bu \rangle|  \ = \ \sqrt{\frac{2}{\pi}} \left(\bu^T\bC^T\bC\bu\right)^{1/2} \ .
\end{eqnarray}
Using the symmetrization and contraction inequalities from Proposition 5.2 in \cite{vershynin14}, we have the bound
\begin{eqnarray}
\E \sup_{\bu \in T} \left|\frac{1}{n}\sum_{i=1}^n |\langle \ba_i,\bu \rangle| - \sqrt{\frac{2}{\pi}} \left(\bu^T\bC^T\bC\bu\right)^{1/2} \right| & \leq & 2 \, \E\sup_{\bu \in T} \left|\frac{1}{n} \sum_{i=1}^n \varepsilon_i \langle \bb_i,\bC \bu\rangle \right| \\
& = & 2 \, \E\sup_{\bu \in T} \left|\left\langle \frac{1}{n} \sum_{i=1}^n \varepsilon_i  \bb_i,\bC \bu\right\rangle \right| \ ,
\end{eqnarray}
where each $\varepsilon_i$ independently takes values $-1$ and $+1$ with probabilities $1/2$.
Note that  vector \mbox{$\bg := \frac{1}{\sqrt{n}} \sum_{i=1}^n \varepsilon_i  \bb_i \sim {\cal N}(0,{\bf I}_n)$}, thus,
\begin{eqnarray}
2 \, \E\sup_{\bu \in T} \left|\left\langle \frac{1}{n} \sum_{i=1}^n \varepsilon_i  \bb_i,\bC \bu\right\rangle \right| & = & \frac{2}{\sqrt{n}} \E \sup_{\bu\in T} |\langle \bg,\bC\bu\rangle| \ = \  \frac{2}{\sqrt{n}} \E \sup_{\bu\in T} |\langle \bC^T\bg,\bu\rangle| .
\end{eqnarray}
This completes the proof.

\end{proof}

\subsubsection{Estimation from Noisy Linear Observations}
Theorem~\ref{egm} can be used to derive error bounds for estimating signals known to belong to a certain subset (sparse sets are a special case we will consider in the next section).
Let $\K\subset \R^p$ be given.
Suppose that we observe
\begin{eqnarray}
\by \ = \ \bA \bx^\star + \bnu \ , \ \ \ \ \frac{1}{n}\|\bnu\|_1 \ \leq \ \varepsilon \ ,
\end{eqnarray}
where $\bx^\star \in \K$.
The following theorems are straightforward extensions of
Theorems 6.1 and 6.2 in \cite{vershynin14}.  We include the proofs for the sake of completeness.

\begin{theorem}
\mbox{\rm (Estimation from noisy linear observations: feasibility program).}  Choose $\widehat{\bx}$ to be any vector satisfying
\begin{eqnarray}
\widehat{\bx} \in \K \mbox{ \ and \ } \frac{1}{n}\|\bA\widehat{\bx}-\by\|_1 \leq \varepsilon \ .
\end{eqnarray}
Then
\begin{eqnarray}
\E \sup_{\bx^\star \in \K} \left\{(\widehat \bx - \bx^\star)^T\bC^T\bC(\widehat \bx - \bx^\star)\right\}^{1/2} & \leq & \sqrt{2\pi} \left( \frac{\E \sup_{\bu \in \K-\K} |\langle \bC^T\bg,\bu\rangle|}{\sqrt{n}} +\varepsilon\right) \ .
\end{eqnarray}
\label{fp}
\end{theorem}

\begin{proof}
We apply Theorem~\ref{egm} to the set $T=\K-\K$ with $2\varepsilon$ instead of $\varepsilon$, which yields
\begin{eqnarray*}
\E \sup_{\bu \in T_{2\varepsilon}} \left(\bu^T\bC^T\bC\bu\right)^{1/2} & \leq & \sqrt{\frac{2\pi}{n}} \, \E \sup_{\bu \in T} |\langle \bC^T\bg,\bu\rangle| \ + \ \sqrt{2\pi} \varepsilon \ ,
\end{eqnarray*}
From here, all we need to show is that for any $\bx^\star \in \K$
\begin{eqnarray}
\widehat\bx -\bx^\star \in T_{2\varepsilon} \ .
\end{eqnarray}
To see this, note that $\widehat\bx,\bx^\star \in \K$, so $\widehat\bx-\bx^\star \in \K-\K = T$.  By the triangle inequality,
$$\frac{1}{n}\|\bA(\widehat\bx-\bx^\star)\|_1 = \frac{1}{n}\|\bA\widehat\bx - \by+\bnu\|_1 \leq \frac{1}{n}\|\bA\widehat\bx -\by\|_1 + \frac{1}{n}\|\bnu\|_1 \leq 2\varepsilon \ ,$$
showing that $\bu = \widehat\bx-\bx^\star$ indeed satisfies the constraints that define $T_{2\varepsilon}$ in (\ref{te}).
\end{proof}

Next we derive an optimization program for the solution. The {\em Minkowski functional} of $\K$ is defined as
$$\|\bx\|_{\K} \ = \ \inf\{\lambda >0 \, : \, \lambda^{-1}\bx\in \K\} \ . $$
If $\K$ is a compact and origin-symmetric convex set with non-empty interior, then $\|\bx\|_{\K}$ is a norm on $\R^p$ \cite{rock}.  Note that $\bx\in \K$ if and only if $\|\bx\|_{\K} \leq 1$.

\begin{theorem}
\mbox{ \rm (Estimation from noisy linear observations: optimization program).}  Choose $\widehat{\bx}$ to be a solution to the the optimization
\begin{eqnarray}
\min\|\bx\|_{\K} \mbox{ \ subject to \ } \frac{1}{n}\|\bA\bx -\by\|_1 \leq \varepsilon \ .
\end{eqnarray}
Then
\begin{eqnarray}
\E \sup_{\bx^\star \in \K} \left\{(\widehat \bx - \bx^\star)^T\bC^T\bC(\widehat \bx - \bx^\star)\right\}^{1/2} & \leq & \sqrt{2\pi} \left( \frac{\E \sup_{\bu \in \K-\K} |\langle \bC^T\bg,\bu\rangle|}{\sqrt{n}} +\varepsilon\right) \ .
\label{width}
\end{eqnarray}
\label{op}
\end{theorem}

\begin{proof}
If we show that $\widehat \bx \in \K$, then the result follows from Theorem~\ref{fp}. Note that the constraint of the program guarantees that $\frac{1}{n} \|\bA\widehat\bx-\by\|_1 \leq \varepsilon$ and by assumption we have $\frac{1}{n}\|\bA\bx^\star-\by\|_1 = \frac{1}{n}\|\bnu\|_1 \leq \varepsilon$.  Thus we have
$$\|\widehat \bx\|_{\K} \ \leq \ \|\bx^\star\|_{\K} \ \leq \  1 \ , $$
since $\bx^\star\in \K$.  The inequality $\|\widehat \bx\|_{\K} \leq 1$ implies that $\widehat\bx \in \K$.
\end{proof}

\subsubsection{Sparse Recovery via OWL}

Recall the definition of the OWL norm
$$\Omega_{\bm w}(\bx) \ = \ \sum_{i=1}^p w_i |x|_{[i]} \ , $$
where $|x|_{[1]},\dots,|x|_{[p]}$  are the magnitudes of the elements of $\bx$ in decreasing order and $w_1\geq w_2\geq \cdots\geq w_p$ is a non-increasing sequence of weights.  The OWL norm satisfies
$$\bar w \, \|\bx\|_1 \ \leq \Omega_{\bm w}(\bx) \ \leq \ w_1\|\bx\|_1 \ , $$
where $\bar w := \frac{1}{N}\sum_{i=1}^N w_i$.  This is easily verified
by minimizing or maximizing the OWL norm subject to a fixed $\ell_1$ norm.
We now prove Theorem~\ref{thm1}.

\begin{proof}(Theorem~\ref{thm1})
Since the signal generating the measurements is assumed to satisfy $\|\bx^\star\|_1\leq \sqrt{s}$, we first need to construct an OWL ball that contains all $\bx \in \R^p$ with $\|\bx\|_1\leq \sqrt{s}$.
Let $\K = \{\bx\in \R^p \, : \, \Omega_{\bm w}(\bx) \leq w_1 \,\sqrt{s} \}$.  Because $\Omega_{\bm w}(\bx) \leq w_1 \|\bx\|_1$, all vectors satisfying $\|\bx\|_1 \leq \sqrt{s}$ belong to $\K$.
Also note that because $\Omega_{\bm w}(\bx)$ is a norm, and $\K$ is a ball of this norm, the Minkowski functional  $\|\bx\|_{\K}$ is proportional to $\Omega_{\bm w}(\bx)$.

The quantity $\E \sup_{\bu \in \K-\K} |\langle \bC^T\bg,\bu\rangle|$ in (\ref{width}), called the width of $\K$, satisfies
$$\E \sup_{\bu \in \K-\K} |\langle \bC^T\bg,\bu\rangle| \ = \ \E \sup_{\bu \in \K-\K} |\langle \bg,\bC\bu\rangle| \ . $$ Note that $$\|\bC\bu\|_1 \ \leq \ \|\bC\|_1 \|\bu\|_1 \ \leq \ \|\bC\|_1 \, \frac{1}{\bar w} \, \Omega_{\bm w}(\bu)  \ . $$
The triangle inequality and the definition of $\K$ imply that for any $\bu \in \K-\K$, $\Omega_{\bm w}(\bu)   \leq 2 w_1\sqrt{s}$, thus we have
$$\|\bC\bu\|_1 \ \leq \ 2\, \|\bC\|_1 \, \frac{w_1}{\bar w} \, \sqrt{s} \ . $$
The width can be then bounded as
\begin{eqnarray*}
\E\sup_{\bu \in \K-\K}|\langle \bg,\bC\bu\rangle| & \leq &  \E\sup_{\{\bv \, : \,  \|\bv\|_1 \leq 2\, \|\bC\|_1 \, \frac{w_1}{\bar w} \, \sqrt{s} \}} |\langle \bg,\bv\rangle|  \ .
\end{eqnarray*}
The $\bv$ that maximizes the right hand side places mass $2\, \|\bC\|_1 \, \frac{w_1}{\bar w} \, \sqrt{s}$ on the largest element of $\bg$ (in magnitude) and zero on every other element. This yields the bound
\begin{eqnarray}
\E\sup_{\bu \in \K-\K}|\langle \bg,\bC\bu\rangle|
& \leq & 2\, \|\bC\|_1 \, \frac{w_1}{\bar w} \, \sqrt{s} \ \E\max_{i=1,\dots,q}|g_i| \ .
\label{b}
\end{eqnarray}
Using Jensen's inequality, the square of the expectation in (\ref{b}) can be bounded as
\begin{eqnarray*}
\left(\E\max_{i=1,\dots,q}|g_i|\right)^2 & \leq & \E \max_{i=1,\dots,q}|g_i|^2 \ \leq
\ \left(\sqrt{2\log q}+1\right)^2 \ ,
\end{eqnarray*}
where the second inequality comes from a chi-square tail bound
(see Lemma 3.2 in \cite{rao12}).  Note that since $q>1$,  $\sqrt{2\log q}+1 < 2\sqrt{2\log q}$.  Putting everything together, we obtain the bound
\begin{eqnarray*}
\E\sup_{\bu \in \K-\K}|\langle \bg,\bC\bu\rangle| & \leq &  4\sqrt{2} \, \|\bC\|_1 \, \frac{w_1}{\bar w} \, \sqrt{s\log q}
\end{eqnarray*}
Theorem~\ref{thm1} now follows immediately from Theorem~\ref{op}, above.
\end{proof}

\section{Conclusion}
In this paper, we have studied {\it ordered weighted} $\ell_1$ (OWL) regularization for sparse estimation problems with strongly correlated variables.  We have proved sufficient conditions under which the OWL regularizer clusters the coefficient estimates, based on the correlation/colinearity of the variables in the design matrix. We have also characterized the statistical performance of OWL regularization for generative models in which certain clusters of regression variables are strongly (even perfectly) correlated, but variables in different clusters are uncorrelated.  Essentially, we showed that, by using OWL regularization, we pay no price (in terms of the number of measurements) for the presence of strongly correlated variables. Future work will include the experimental evaluation of OWL regularization and its application to other problems, such as logistic regression.

\end{document}